\documentclass[letterpaper]{article} 
\usepackage{aaai24}  
\usepackage{times}  
\usepackage{helvet}  
\usepackage{courier}  
\usepackage[hyphens]{url}  
\usepackage{graphicx} 
\usepackage{natbib}  
\usepackage{caption} 
\usepackage{algorithm}
\usepackage{algorithmic}

\usepackage{amsmath}
\usepackage{amssymb}
\usepackage{mathtools}
\usepackage{amsthm}
\usepackage{dsfont}
\usepackage{multirow}
\usepackage{subfigure}
\usepackage{comment}
\newtheorem{theorem}{Theorem}
\newtheorem{lemma}{Lemma}
\newtheorem{corollary}{Corollary}
\theoremstyle{definition}

\newtheorem{assumption}{Assumption}
\newtheorem{remark}{Remark}

\title{Forced Exploration in Bandit Problems}
\author{
    Han Qi,
    Fei Guo,
    Li Zhu
    \thanks{Li Zhu is the corresponding author.}
}
\affiliations{

    School of Software Engineering, Xi'an Jiaotong University\\
    28 Xianning West Road\\
    Xi'an, Shaanxi, 710049, China\\
    \{qihan19,co.fly\}@stu.xjtu.edu.cn, zhuli@xjtu.edu.cn
}

\usepackage{bibentry}

\begin{document}

\maketitle

\begin{abstract}
The multi-armed bandit(MAB)  is a classical sequential decision problem. Most work requires assumptions about the reward distribution (e.g.,  bounded), while practitioners may have difficulty obtaining information about these distributions to design models for their problems, especially in non-stationary MAB problems. This paper aims to  design a multi-armed bandit algorithm that can be implemented without using information about the reward distribution  while still achieving substantial regret upper bounds. To this end, we propose a novel algorithm alternating between greedy rule and forced exploration. Our method can be applied to Gaussian, Bernoulli and other subgaussian distributions, and its implementation does not require additional information. We employ a unified analysis method for different forced exploration strategies and provide problem-dependent regret upper bounds for stationary and piecewise-stationary settings. Furthermore, we compare our algorithm with popular bandit algorithms on different reward distributions.
\end{abstract}

\section{Introduction}

The multi-armed bandit (MAB) is a classical reinforcement learning problem. It simulates the process of pulling different arms of a slot machine, where each arm has an unknown reward distribution and only the selected arm's reward is observed. The learner aims to find the optimal policy that maximizes the cumulative reward. To achieve better long-term rewards, the learner must balance exploration and exploitation. 

MAB has been widely used in many sequential decision tasks, such as online recommendation systems \cite{li2011unbiased,li2016collaborative}, online advertisement campaign \cite{schwartz2017customer} and  diagnosis and treatment experiment \cite{vermorel2005multi}. Most of the existing multi-armed bandit algorithms are based on Upper Confidence Bound (UCB) \cite{auer2002finite} and Thompson Sampling (TS) \cite{thompson1933likelihood}, which often rely on assumptions about the reward distribution. For example, some studies assume that the reward distribution follows a sub-Gaussian distribution, and the algorithm implementation requires knowledge of the variance parameter. Additionally, some works assume that the reward distribution is bounded, and the algorithm design necessitates specific upper-bound information.

In the standard MAB model, the reward distribution remains constant. However,  in real-world scenarios,  the distribution of rewards may vary over time. In the context of clinical trials, the target disease may undergo mutations, causing the initially optimal treatment to potentially become less effective compared to another candidate \cite{gorre2001clinical}. Similarly, in online recommendation systems, user preferences are prone to evolve \cite{wu2018learning}, rendering collected data progressively outdated. During the past ten years, several works have been conducted on non-stationary multi-armed bandit problems. These approaches can be broadly classified into two categories: one category involves detecting changes in the reward distribution using change-point detection algorithms \cite{liu2018change,cao2019nearly,auer2019adaptively,chen2019new,besson2022efficient}. In contrast, the other category focuses on mitigating the impact of past observations in a passive manner \cite{garivier2011upper,raj2017taming,trovo2020sliding}.
Among them, \citet{auer2019adaptively,chen2019new,besson2022efficient}  can derive regret bounds without knowing the number of changes. However, as in the stationary settings, these algorithms also require prior information on the reward distributions for their implementation.

Recently, the non-parametric multi-armed bandit algorithm has garnered considerable attention \cite{lattimore2017scale,kveton2019perturbed,kveton2019garbage,riou2020bandit,baudry2021limited,liu2022extended}. The implementation of these algorithms does not require strong parametric assumptions on the reward distributions, that is, a single implementation can be applied to several different reward distributions. Although the implementation of some mentioned algorithms \cite{kveton2019garbage,kveton2019garbage,riou2020bandit} does not need to know the information of reward distribution in advance, they require that the reward distribution is bounded. Therefore, they cannot be applied to a wider unbounded distribution. \citet{wang2020residual} propose a  perturbation based exploration method called \textit{ReBoot} which has been analyzed only for Gaussian distributions. 
\citet{chan2020multi} propose SSMC by comparing
the subsample means of the leading arm with the sample means of its competitors for one-parameter exponential families distributions. However, as with the existing subsampling algorithms \cite{baransi2014sub}, they have to store
the entire history of rewards for all the arms.
\citet{baudry2021limited} propose LB-SDA and LB-SDA-LM with a deterministic subsampling rule for one-parameter exponential families distributions. LB-SDA-LM is introduced to address the issue of high storage space and computational complexity. By using a sliding window, LB-SDA can be applied to non-stationary settings. While it cannot be applied to Gaussian distributions with unknown variance and the arms must come from the same one-parameter exponential family. \citet{lattimore2017scale} propose an algorithm that can be applied to various common distributions but requires the learner to know the specific kurtosis.
\citet{liu2022extended} propose the extended robust UCB policies without knowing the knowledge of an upper bound on specific moments of reward distributions. However, their algorithm needs a specific moment control coefficient as input. 

Table \ref{table:1} lists the required assumptions and preliminary information for some mentioned algorithms. Our algorithm is applicable to sub-Gaussian distributions. This assumption covers common distributions such as Gaussian, bounded, and Bernoulli distributions. The implementation of our algorithm does not require knowledge of the variance parameter for sub-Gaussian distributions.

\noindent {\bfseries{Contributions }} In this paper, we propose a  bandit algorithm that  achieves respectable upper bounds on regret without using the parameters of  distribution model  for implementation. Our method is simple and easy-to-implement with the core idea of forcing exploration. Each time step forces an arm to be explored or uses greedy rule to select an arm. Specifically, our algorithm takes a non-decreasing sequence $\{f(r)\}$ as input. This input sequence controls how each arm is forced to be pulled. 



In the stationary settings, we provide problem-dependent regret upper bounds. This regret upper bound is related to the input sequence, i.e., different input sequence will lead to different upper bounds. For example, if $\{f(r)\}$ is taken as an exponential sequence, our algorithm can guarantee a problem-dependent asymptotic upper bound $O((\log T)^2)$ on the expected number of pulls of the suboptimal arm. 

In the piecewise-stationary settings, we use a sliding window $\tau$  along with a reset strategy to adapt to changes in the reward distribution. The reset strategy is realized by periodically resetting the sequence $\{f(r)\}$. This ensures that the algorithm maintains its exploration capability after $\tau$ time steps. 
Furthermore,  we show that our algorithm  has the ability to achieve a  problem-dependent asymptotic regret bound of order $\tilde{O}(\sqrt{TB_T})$ if the number of breakpoints $B_T$ is a constant independent of $T$. This asymptotic regret bound matches the lower bounds in finite time horizon
\cite{garivier2011upper}, up to logarithmic factors.


\begin{table*}[!htp]
	\caption{Assumptions and implementation parameters
    }
	\centering
	\begin{tabular}{|c|c|c|c}
		\hline
		Algorithms& Assumption		&  Known  \\ \hline
		UCB,TS& Support in $[a,b]$ or $\sigma$-subgaussian & $a,b$ or $\sigma$ \\
        \hline
        Robust-UCB&  $\mathbb{E}[|X-\mu|^{1+\epsilon}]<u,0<\epsilon\leq 1$ & $\epsilon,u$ \\
        \hline
        ReBoot\cite{wang2020residual}& Gaussian & -\\
        \hline
        SSMC,LB-SDA & one-parameter exponential families& - \\
        \hline
        \cite{lattimore2017scale}& Bounded kurtosis $\mathbb{E}[(\frac{X-\mu}{\sigma})^4]< \kappa$& $ \kappa$\\
        \hline
        \cite{liu2022extended}& - & moment control coefficient\\
        \hline
        Our method & subgaussian & -\\
        \hline

	\end{tabular}
	\label{table:1}
\end{table*}
\section{Problem Formulation}

{\bfseries{Stationary environments}} 
Let's consider a multi-armed bandit problem with finite time horizon $T$ and a set of arms $\mathcal{A}:= \{1,...,K\}$. At time step $t$, the learner must choose an arm $i_t \in \mathcal{A}$ and receive the corresponding reward $X_t(i_t)$. The reward comes from a different distribution unknown to the learner and the expectation of $X_t(i)$ is denoted as $\mu(i) = \mathbb{E}[X_t(i)].$ Let $i^{*}$ denote the expected reward of the optimal arm, i.e. ,$\mu(i^{*})=\max_{i \in \{ 1,...,K\}}\mu(i)$. The gap between the expectation of the optimal arm and the suboptimal arm is denoted as $\Delta(i)=\mu(i^{*})-\mu(i)$.

The history $h_t$ is defined as the sequence of actions and rewards from the previous $t-1$ time steps. A policy, denoted as $\pi$, is a function $\pi(h_t) = i_t$ that selects arm $i_t$ to play at time step $t$ based on the history $h_t$. The performance of a policy $\pi$ is measured in terms of cumulative expected regret:
\begin{equation}
    \label{regret}
    R^{\pi}_T=\mathbb{E}\left[\sum_{t=1}^{T} \mu(i^{*})-\mu(i_t) \right].
\end{equation}
where $\mathbb{E}[\cdot]$ is the expectation with respect to randomness of $\pi$. Let $k_T(i)=\sum_{t=1}^{T}\mathds{1}\{ i_t=i,\mu(i)\neq \mu(i^{*}) \}$, the regret can be denoted as 
\begin{equation}
    R^{\pi}_T=\sum_{i=1}^{K}\Delta(i)\mathbb{E}[k_T(i)]
\end{equation}
In later section, we provide a regret upper bound for our method by analyzing the upper bound on $\mathbb{E}[k_T(i)]$.

\noindent {\bfseries{Piecewise-stationary Environment}}
The piecewise-stationary bandit problem has been extensively studied. Piecewise-stationary bandits pose a more challenging problem as the learner needs to balance exploration and exploitation within each stationary phase and during the changes between different phases. In this setting, the reward distributions remain constant for a certain period and change at unknown time steps, called \textit{breakpoints}. The number of breakpoints is denoted as $B_T=\sum_{t=1}^{T-1} \mathds{1}\{ \exists i \in \mathcal{A}:\mu_t(i) \neq \mu_{t+1}(i) \}$. The optimal arm may vary over time and is denoted by $\mu_t(i^{*}_t)=\max_{i \in \{1,..,K \}}\mu_t(i) $. 
The performance of a policy is measured by
\begin{equation}
    R^{\pi}_T=\mathbb{E}\left[\sum_{t=1}^{T} \mu_t(i_t^{*})-\mu_t(i_t) \right].
\end{equation}
As in the stationary environment, the regret can be analyzed by bound the expectation of the number of pulls of suboptimal arm $i$ up to the time step $T$.

\section{Stationary Environments}

\subsection{Forced Exploration} Our algorithm pulls  arms in two ways. The first is based on the greedy rule, which pulls the arm with the largest value of the estimator. The second one is a forced exploration step. Specifically, the algorithm takes a sequence $\{ f(1),f(2),... \}$as  input.
Let $p(i)$ denote  the number of times that  arm $i$ is not pulled. At round $r$, if $p(i) \geq f(r)$, arm $i$ will be forced to pull once and reset $p(i)$ to 0.
If all arms have been pulled at least once at round $r$, let $r=r+1$ and repeat the process in the next round. 

Algorithm 1 shows the pseudocode of our method. We require the input sequence to be non-decreasing to prevent the algorithm from over-exploring. If $f(r)$ is less than a small constant $c$ after $t$ time steps, 

\begin{equation}
    \label{linear_regret}
    R^{\pi}_T > (K-1)\frac{T-t}{c}.
\end{equation}
This shows that the algorithm can only obtain linear regret. Step 3 is the greedy rule. In Step 5, the algorithm pull the arm that has not been pulled more than $f(r)$ times.
To simplify the notation, we use the equivalent notation of $\arg\max_i p(i)$. 

\noindent {\bfseries{Related to Explore-Then-Committee (ETC)}} \cite{garivier2016explore} study ETC policy for two-armed bandit problems with Gaussian rewards. ETC explores until a stopping time $s$. Let \[f(r)=\left\{ \begin{aligned}
    1, r \leq s\\
    0,r > s
\end{aligned} \right. ,\] our method is equivalent to ETC. The essential difference  is that ETC is non-adaptive, i.e., for different problem instances, the set of all exploration
rounds and the choice of arms therein is fixed before the first round.
For some common input sequences, such as $f(r) = \sqrt{T}$  or $f(r) = r$, the exploration schedule of our method is related to the history of pulled arms, so our method is adaptive.

\noindent {\bfseries{Related to $\epsilon_t$-greedy}}
$\epsilon_t$-greedy strategy  tosses a coin with a success probability $\epsilon_t$ at time step $t$ and randomly selects an arm if  success, otherwise the arm with the largest average reward is selected. If the exploration probability satisfies $ \epsilon_t \sim t^{-\frac{1}{3}}$, this greedy strategy achieves regret bound on the order of $O(T^{\frac{2}{3}})$.
Similar to our method, this strategy needs to decide whether to explore at each time step. The difference is that $\epsilon_t$-greedy is non-adaptive since it does not adapt their exploration schedule to the history.

\begin{algorithm}[tb]
\caption{FE}
\label{alg:algorithm}
\textbf{Input}: non-decreasing sequence $\{f(r)\}$, $K$ arms, horizon $T$\\
\textbf{Initialization:} $t=1,r=0,f(0)=0,\forall i \in \{1,...,K\},p(i)=0,flag(i)=0$
\begin{algorithmic}[1] 
\WHILE{$t <T$}
\IF {$\forall i \in \{1,...,K\},p(i)<f(r)$}
\STATE Pull arm $i_t=\arg \max_i \hat{\mu}(i)$.
\ELSE
\STATE Pull arm $i_t= \arg \max_i p(i)$
\ENDIF
\STATE Update the estimate $\hat{\mu}(i_t)$,$p(i_t)=0,flag(i_t)=1$\\
\STATE $p(i) = p(i)+1$ for all unpulled arm $i$
\IF {$\forall i \in \{1,...,K\},flag(i)==1$ }
\STATE $r=r+1$\\
\STATE $\forall i \in \{1,...,K\},flag(i)=0$
\ENDIF
\STATE $t=t+1$
\ENDWHILE
\end{algorithmic}
\end{algorithm}

\subsection{Regret Analysis}
In this section, we bound the expectation of the number of pulls of suboptimal arms. 
The detailed proofs of Theorem  and Corollary  are given in the Appendix. 

Let $h_t(i)$ denotes the number of times  arm $i$ is forced to pull (Step 5) up to time $t$. We give the following theorem and proof sketch.  

\begin{theorem}
    \label{theorem1}
    Assume that the reward distribution is $\sigma$-subgaussian. Let $\{f(r)\}
    $ be a non-decreasing sequence,  
    for any suboptimal arm $i$,
    \begin{equation}
        \mathbb{E}[k_T(i)] \leq  h_T(i) 
        + 2m e^{\frac{1}{m}} \sum_{t=1}^{T} e^{\frac{-h_t(i)}{m}},
    \end{equation}
    where $m=\frac{8\sigma^2}{\Delta(i)^2}$. 
\end{theorem}
\noindent {\bfseries{Proof sketch} }  We bound the number of suboptimal arm's pulls in Step 3 and Step 5 in Algorithm \ref{alg:algorithm} respectively. The number of pulls according to the greedy rule  can be estimated using the properties of subgaussian.  
Summing the regret caused by the greedy rule and the forced exploration (which can be denoted as $h_T(i)$) leads to this Theorem.

Next, we give more specific upper bounds for different input sequences. 

\begin{corollary}[Constant sequence]
    \label{corollary1}
    Let $f(r)=\sqrt{T}$. We have $h_t(i) \in [\frac{t}{\sqrt{T}}-1,\frac{t}{\sqrt{T}}+1]$. Therefore, 
    \begin{equation}
        \mathbb{E}[k_T(i)] \leq \sqrt{T}(1+2m^2e^{\frac{2}{m}})+1.
    \end{equation}
\end{corollary}

\begin{corollary}[Linear sequence]
    \label{corollary2}
    Let $f(r)=r$.  We have,
    \begin{equation}
        \mathbb{E}[k_T(i)] \leq \sqrt{2T} +K^2 
    +6m^3e^{\frac{3}{m}}
    \end{equation}
\end{corollary}

\begin{corollary}[Exponential sequence]
    \label{corollary3}
    Let $f(r)=a^r(a>1)$. 
    We have,
    \begin{equation}
        \label{exp}
        \begin{aligned}
            \mathbb{E}[k_T(i)] & \leq\frac{\log(T(a-1)+1)}{\log(a)} +(K+1)\frac{\log(K+1)}{\log(a)}\\ &+ 2me^{\frac{1}{m}}\sum_{t=1}^{T} (1+\frac{a-1}{K+1}t)^{-\frac{1}{m\log(a)}} .
            \end{aligned}
    \end{equation}
\end{corollary}

\subsubsection{Lower Bounds}
\label{lowerbound}
Many studies have demonstrated the lower bounds of regret for the stationary multi-armed bandit problem \cite{slivkins2019introduction,lai1985asymptotically,bubeck2013bandits}. A commonly used problem-dependent lower bound for bounded rewards or subgaussian rewards with variance parameter $1$  is 
\begin{equation}
    \label{lower}
    \liminf_{T \to \infty }
    \frac{R^{\pi}_T}{\log (T)} \geq \sum_{i:\Delta(i)\neq 0}\frac{C_{\mathcal{I}}}{\Delta(i)},
\end{equation}
where $C_{\mathcal{I}}$ is a problem-dependent constant. 

Define $f^{-1}(r)$ as $\min \{x:f(x)\geq r \} $. Let $t_0$ denote the number of all time steps satisfying $f(r) \leq K$, our method has a simple problem-dependent lower bound: 
\begin{equation}
    \label{eq1}
    \max \{ n:\sum_{r=f^{-1}(K+1)}^{n}f(r) \leq T-t_0 \}.
\end{equation}
Notethat, this lower bound  only  considers forced exploration and does not take into account the regret incurred by the greedy rule. The true regret lower bound is larger than the one computed by Equation \ref{eq1}.

If $f(r)$ is constant sequence ($f(r)=\sqrt{T}$) or linear sequence ($f(r)=r$), this problem-dependent lower bound is $\Omega (\sqrt{T})$. According to Corollary \ref{corollary1} and Corollary \ref{corollary2}, the lower bound of constant sequence and linear sequence is on the same order of the upper bound.

If $f(r)$ is  exponential sequence ($f(r)=a^r$),  this lower bound is $\Omega(\frac{\log (T)}{\log(a)})$.

\subsubsection{Upper Bound of Exponential Sequence }
From the lower bounds of the above three sequences, the regret upper bound of exponential sequence can be expected to reach the lower bound (Equation \ref{lower}). However, it often fails to achieve this goal.
Since there is lack of  knowledge about $\sigma$ and $\Delta(i)$, we can't tune the parameter $a$ to obtain an optimal upper bound in Equation \ref{exp}.  

If there is no information about rewards distributions,  there are two simple ways to set the exponential sequence: 
\begin{itemize}
    \item  $a$ is  a small constant independent of $T$. If the value of $m=\frac{8\sigma^2}{\Delta(i)^2} $ is appropriate such that 
    \begin{equation}
        \label{temp}
        m\log(a) < 1,
    \end{equation}
    we have 
    \[\mathbb{E}[k_T(i)] = O(me^{1/m}\log (T)).\] 
    This upper bound is optimal with respect to the order of $T$. However, in general, we cannot guarantee  the parameter $a$ can satisfy Equation \ref{temp}. 
    \item  $a$  is associated with  $T$, such as $a=e^{\frac{1}{\log (T)}}$. Then, $f(r) = e^{\frac{r}{\log (T)}}$,  we can get the asymptotic regret
    \begin{equation}
        \label{temp1}
        \mathbb{E}[k_T(i)] = O(me^{\frac{1}{m}} (\log(T))^2).
    \end{equation}
    Like the other two sequences, this problem-dependent  asymptotic upper bound also matches the order  of lower bound. 
    Our upper bound is not optimal and there has been work to show that optimality is impossible. For example,  recently \citet{agrawal2021optimal,ashutosh2021bandit} have proved the impossibility of a problem-dependent logarithmic regret for light-tailed distributions without
    further assumptions on the tail parameters.
\end{itemize}

\begin{remark}
\label{remark1}
We use an example to illustrate Corollary \ref{corollary3} and Equation \ref{temp1}.
Assume that the reward distribution follows a Gaussian distribution with variance $1$.  If $T$ is sufficiently large, it holds that $m\log(a)=\frac{m}{\log (T)} < 1 $. If $m >1$, we get 
\begin{equation}
    \mathbb{E}[k_T(i)] \leq (K+2)(\log (T))^2+e (K+1)\frac{16(\log(T))^2}{(\Delta(i))^2}.
\end{equation}
If $m \leq 1$, the above equation holds obviously.
We can derive the following  regret upper bound:
\begin{equation}
    \begin{aligned}
    R^{\pi}_T &\leq 8\sqrt{e}(K+1)\sqrt{T}\log(T) \\
    &+ (K+2)(\log(T))^2\sum_{i=1}^{K}\Delta(i).
    \end{aligned}
\end{equation}
The above regret bound matches the optimal upper bounds $ \tilde{O}(\sqrt{T})$ in stationary multi-armed bandits. The extra terms $\log(T)$ and $K$ are due to the forced exploration.
Note that, this regret upper bound  also holds asymptotically.
\end{remark}

\section{Non-Stationary Environments}
In this section, we consider the piecewise-stationary settings. We employ a method often used in non-stationary bandits problems - sliding windows. One might think that all we need is to add sliding windows to the mean estimator $\hat{\mu}$. However, this does not work in non-stationary situations. Consider an input sequence $f(r)$ that can be incremented to $+\infty$. If $f(r) > T$ holds after some time step $t$, then between the time steps in $[t,T]$, the algorithm will not force to explore any arm but only pulls the arm based on the value of the mean estimator. If the reward distribution changes, the algorithm will suffer from very poor performance. To keep the exploration ability  when the reward distribution changes,  we propose  to periodically reset the exploration sequence.

\subsection{Reset $\{f(r)\}$} 
Define 
\[\hat{\mu}_t(\tau,i) =\frac{1}{N_t(\tau,i)}\sum_{s=t-\tau+1}^{t}X_s(i)\mathds{1}\{i_s=i\},\]
\[  N_t(\tau,i)=\sum_{s=t-\tau+1}^{t}\mathds{1}\{i_s=i\}.
\]
$\hat{\mu}_t(\tau,i)$ denote the sliding window estimator, which using only the $\tau$ last pulls. We reset $r=1$ every $\tau$ time steps  to  make the input sequence re-grow  from $f(1)$. This simple reset strategy ensures that each arm is forced to be explored a certain number of times in each $\tau$ interval.

The size of this reset interval is set to $\tau$ for simplicity and ease of analysis. One can use  intervals of other sizes. For our method, the selection of the reset interval should ensure that each arm can be forced to explore a certain number of times within the interval $[t-\tau+1,t]$. 
Algorithm \ref{alg:algorithm2} shows the pseudocode of our method for non-stationary settings. Compared to Algorithm \ref{alg:algorithm}, only a sliding window is added to the estimator and $\{f(r)\}$ is reset periodically.

		

\begin{algorithm}[tb]
\caption{SW-FE}
\label{alg:algorithm2}
\textbf{Input}: non-decreasing sequence $\{f(r)\}$,sliding window $\tau$, $K$ arms, horizon $T$\\
\textbf{Initialization:} $t=1,r=0,f(0)=0,\forall i \in \{1,...,K\},p(i)=0,flag(i)=0$
\begin{algorithmic}[1] 
\WHILE{$t <T$}
\IF {$\forall i \in \{1,...,K\},p(i)<f(r)$}
\STATE Pull arm $i_t=\arg \max_i \hat{\mu}_t(\tau,i)$.
\ELSE
\STATE Pull arm $i_t= \arg \max_i p(i)$
\ENDIF
\STATE Update the estimate $\hat{\mu}_t(\tau,i_t)$,$p(i_t)=0,flag(i_t)=1$\\
\STATE $p(i) = p(i)+1$ for all unpulled arm $i$
\IF {$\forall i \in \{1,...,K\},flag(i)==1$ }
\STATE $r=r+1$\\
\STATE $\forall i \in \{1,...,K\},flag(i)=0$
    
\ENDIF
\IF { $t$ mod $\tau==0$}
    \STATE $r=1$
\ENDIF
\STATE $t=t+1$
\ENDWHILE
\end{algorithmic}
\end{algorithm}

\subsection{Regret Analysis}

In  non-stationary setting, more regret is incurred compared to the stationary setting. This is the cost that must be paid to adapt to changes in the reward distributions. For our approach, the regret that arises more than the stationary environment comes from two aspects. The first is that due to the use of sliding window, the historical data used to estimate arm expectations are limited to at most $\tau$. The second, which is unique to our method, is that the number of times the suboptimal arm is pulled increases due to the periodically
resetting of the exploration sequence.

Let $h_t(\tau,i)$ denote the number of forced pulls for arm $i$ in the $\tau$ last plays. Since we use the sliding window and a reset strategy, $h_t(\tau,i)$ will first increase and then change within a certain range.
Let $ \Delta_T(i)= \min \{ \Delta_t(i): i\neq i_t^{*},t\leq T\}$, be the minimum difference between the expected reward of the best arm $i_t^{*}$ and the expected reward of arm $i$ in all times $T$ when arm $i$ is not the best arm. 

\begin{theorem}
    \label{theorem2}
    Assume that the reward distribution is $\sigma$-subgaussian. Let  $\{f(r)\}
    $ be an non-decreasing sequence,$\tau$ is the sliding window, for any suboptimal arm $i$,
    \begin{equation}
        \begin{aligned}
        \mathbb{E}[k_T(i)] &\leq \frac{T}{\tau}\left (h_{\tau}(\tau,i)
        + m_T e^{\frac{1}{m_T}} \sum_{t=1}^{\tau} e^{\frac{-h_t(\tau,i)}{m_T}}\right )\\
        &+\frac{T}{\tau}(1+2m_T\log(\tau))+B_T\tau,
        \end{aligned}
    \end{equation}
    where  $m_T=\frac{8\sigma^2}{\Delta_T(i)^2}$. 
\end{theorem}
\noindent {\bfseries{Proof sketch} }  The proof is adapted from the analysis of SW-UCB \cite{garivier2011upper}. The regret comes from three aspects:  greedy rules, forced exploration and  the analysis methods of sliding windows.
The forced exploration incurs regrets with upper bounds $ \frac{T}{\tau}h_{\tau}(\tau,i)$. 
The analysis approach of the sliding window itself has a regret upper bound of $B_T\tau$. The regret caused by greedy rule can be estimated by regret decomposition. 
$\sum_{t=1}^{T} \mathds{1}\{ i_t=i \neq i_t^{*}, N_t(\tau,i) < A(\tau) \}$ can be bounded by $\lceil T/\tau \rceil A(\tau)$, $A(\tau)= 2m_T\log(\tau)$. 
We can decompose the regret in the following way:
\[ 
    \begin{aligned}
    &\{ i_t=i \neq i_t^{*}, N_t(\tau,i) > A(\tau) \}  
    \subset \\ 
    & \{ \mu_t(\tau,i) > \mu_t(i) + \frac{\Delta(i)}{2}, N_t(\tau,i)>A(\tau)\} \cup \\&
     \{\mu_t(\tau,i_t^{*}) \leq \mu_t(i_t^{*}) -\frac{\Delta(i)}{2} \}
    \end{aligned} 
\] 
The analysis methods of these two parts are similar to the stationary settings.
Summing over all leads to this Theorem.

Similar to the stationary setting, we provide the specific bounds for some explore sequence. Our method requires that  each arm can  be forced to explore  within $[t-\tau+1,t]$. Constant sequences cannot be set to the same value as $\sqrt{T}$ in the stationary scenario. This could potentially result in a long time steps without exploring the optimal arm. Since the size of the sliding window is $\tau$, we set the constant sequence as $f(r)=\sqrt{\tau}$.

\begin{corollary}
    \label{corollary4}
    Let $f(r)=\sqrt{\tau}$.
    We have,
    \begin{equation}
        \label{lin1}
        \begin{aligned}
        \mathbb{E}[k_T(i)] &\leq B_T \tau +\frac{T}{\tau}(1 +2m_T\log(\tau) +\sqrt{\tau}m_T^2e^{\frac{2}{m_T}}) \\
        &+ 
        \frac{T}{\tau}\left (1+\sqrt{\tau
        }\right )
        \end{aligned}
    \end{equation}
\end{corollary}

\begin{corollary}
    \label{corollary5}
    Let $f(r)=r$.
    We have,
    \begin{equation}
        \label{con1}
        \begin{aligned}
        \mathbb{E}[k_T(i)] &\leq B_T \tau +\frac{T}{\tau}(1 +2m_T\log(\tau) +3m_T^3e^{\frac{3}{m_T}}) \\
        &+ 
        \frac{T}{\tau}\left (K^2+ \sqrt{2\tau
        }\right )
        \end{aligned}
    \end{equation}
\end{corollary}

\begin{corollary}
    \label{corollary6}
    Let $f(r)=a^r(a>1)$.
    We have,
    \begin{equation}
        \label{exp1}
        \begin{aligned}
        \mathbb{E}[k_T(i)] &\leq B_T \tau +\frac{T}{\tau}m_Te^{\frac{1}{m_T}}\sum_{t=1}^{\tau} \left (1+\frac{a-1}{K+1} t\right )^{-\frac{1}{m_T\log(a)}} \\
        &+ 
        \frac{T}{\tau}\left (1+2m_T\log(\tau)+ (K+2)\frac{\log (\tau+1)}{\log(a)}\right )
        \end{aligned}
    \end{equation}
\end{corollary}

\begin{figure*}[!htbp] 
    \centering 
    \subfigure[]{ 
        \begin{minipage}[b]{0.45 \textwidth} 
            \centerline{	\includegraphics[width=8cm]{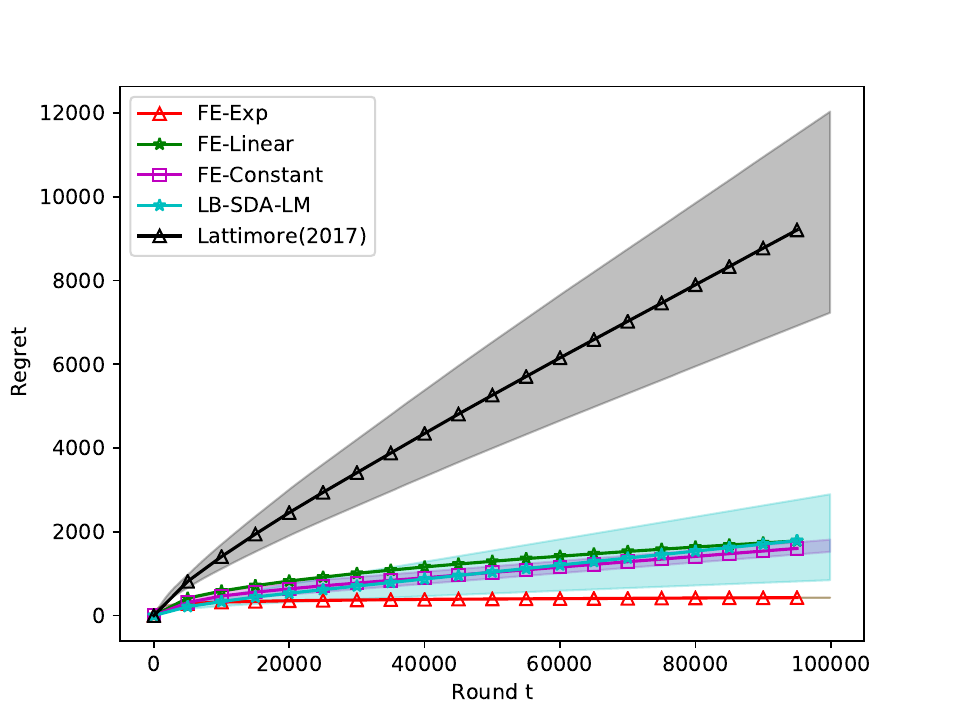}}
        \end{minipage} 
    } 
    \subfigure[]{
        \begin{minipage}[b]{0.45\textwidth}
            
            \centerline{	\includegraphics[width=8cm]{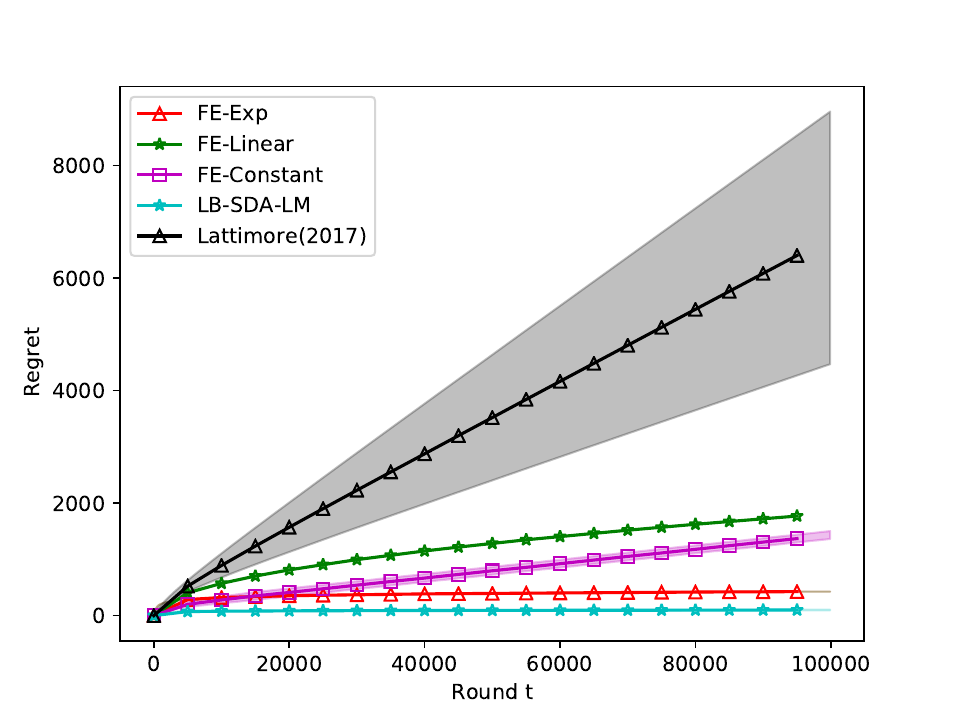} }
        \end{minipage} 	
    }
            
    \caption{Settings with $K=10,T=100000$.  Gaussian rewards (a), Bernoulli rewards (b).   } 
    \label{fig:1}
\end{figure*}
\begin{remark}
    \label{remark2}
   The sliding window methods \cite{garivier2011upper,baudry2021limited} generally suggest that the size of sliding window is $ \tau=\sqrt{T\log(T)/B_T}$. For constant and linear sequence, we get 
   \[
    \mathbb{E}[k_T(i)] = O(T^{\frac{3}{4}}\sqrt{B_T\log(T)}). 
   \]
    For  exponential sequence, we can take $a=e^{\frac{1}{\log(\tau)}}$ similar to stationary settings.  It can be observed that setting the sliding window to $\tau=\sqrt{T/B_T}\log(T) $ for exponential sequence will yield a smaller upper bound, effectively reducing it by $\sqrt{\log(T)}$.
    If the number of breakpoints is constant, we have the following asymptotic bound
\begin{equation}
    \label{ns}
     \mathbb{E}[k_T(i)] = O(\sqrt{TB_T}\log(T)).
\end{equation}

\end{remark}

\section{Experiments}

\subsection{Stationary Settings}
In this section, we compare our method with other non-parametric bandit algorithms on Gaussian and Bernoulli  distribution rewards\footnote{\url{https://github.com/qh1874/Force_Explor}}. Our method is instantiated by three different sequence: FE-Constant, FE-Linear, FE-Exp. They use constant ($f(r)=\sqrt{T}$), linear ($f(r)=r$), and exponential ($f(r)=e^{\frac{r}{\log(T)}}$) sequences, respectively. We compare the above three instances of our method with two  representative non-parametric algorithms: LB-SDA-LM  and Lattimore(2017).
Due to the potentially high time complexity of LB-SDA algorithm, we turn to comparing LB-SDA-LM, an alternative algorithm that achieves the same theoretical results but with much lower complexity. The implementation of Lattimore(2017)  seems challenging, we use a roughly equivalent and efficiently computable alternative \cite{lattimore2017scale}.

The means and variances of Gaussian distributions  are randomly generated from  uniform distribution: 
\[\mu(i) \sim U(0,1), \]\[\sigma(i) \sim U(0,1).\]
The means of Bernoulli  distribution are also generated from $U(0,1)$. 
The time horizon is set as $T=100000$. We fix  the number of arms as $K=10$. 
We measure the performance of each algorithm
with the cumulative expected regret defined in Equation \ref{regret}. The expected regret is averaged on $100$ independently runs. The 95\% confidence interval is obtained by performing $100$
independent runs and is shown as a semi-transparent region in the figure.

Figure \ref{fig:1} shows the results of Gaussian and  Bernoulli rewards. Constant and linear sequences exhibit similar performance. 
The implementation of Lattimore(2017) using the approximation method may lead to significant variance in experimental results, and its performance could be inferior to other methods. LB-SDA-LM is applicable to single-parameter exponential family distributions. This method demonstrates optimal performance on Bernoulli distributions. However, its performance is notably weaker on Gaussian distributions with unknown variance in $(0,1)$. 
Our approach, FE-EXP, although its theoretical upper bound is asymptotic and not optimal, achieves remarkable performance on Gaussian and Bernoulli rewards. 
		

\subsection{Non-stationary Settings}

\begin{figure}[!htbp] 
		
    \centering{\includegraphics[width=8cm]{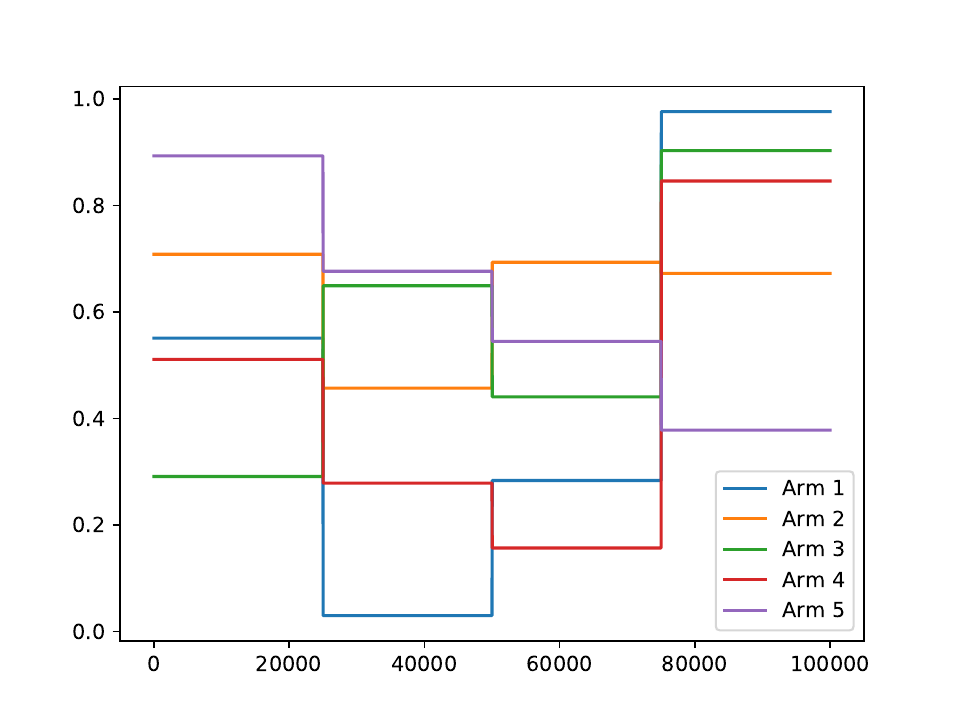}}
    \caption{ $K=5$,$B_T=5$ for Gaussian rewards.} 
    \label{fig:2}
\end{figure}

\begin{figure*}[!htbp] 
    \centering 
    \subfigure[]{ 
        \begin{minipage}[b]{0.45 \textwidth} 
            \centerline{	\includegraphics[width=8cm]{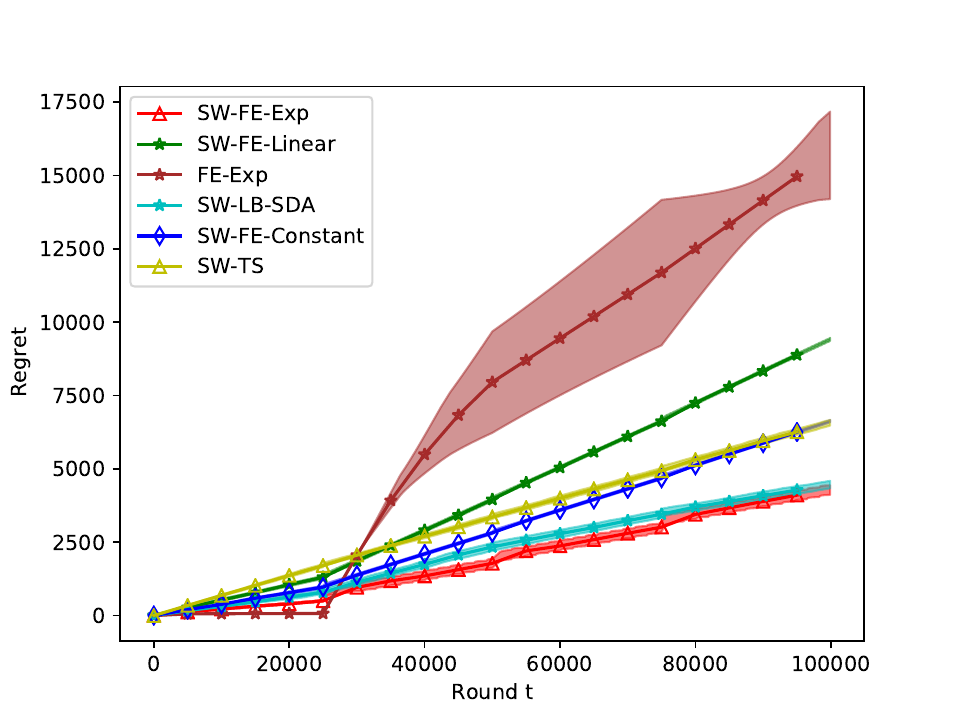}}
        \end{minipage} 
    } 
    \subfigure[]{
        \begin{minipage}[b]{0.45\textwidth}
            
            \centerline{	\includegraphics[width=8cm]{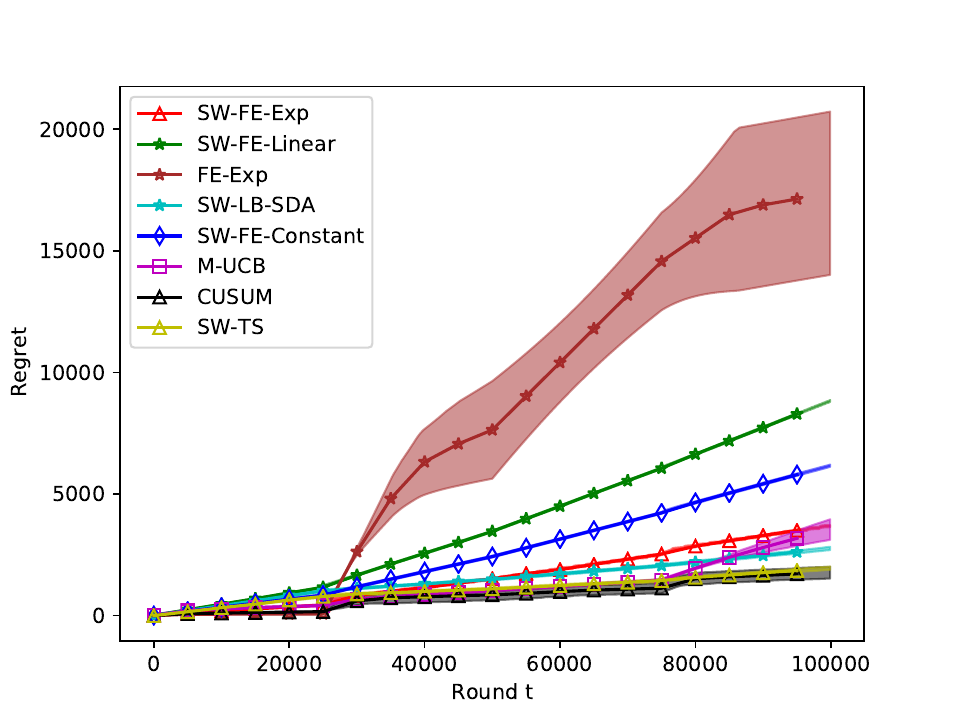} }
        \end{minipage} 	
    }
    \caption{Settings with $K=5,B_T=5,T=100000$.  Gaussian rewards (a), Bernoulli rewards (b).   } 
    \label{fig:3}
\end{figure*}
		

In this section, we compare our method with other non-stationary bandit algorithms. Specifically, our method employs four instances:  constant sequence with sliding window (SW-FE-Constant), linear sequence with sliding window (SW-FE-Linear), exponential sequence with sliding window (SW-FE-EXP), and the exponential sequence without  sliding window (FE-EXP). We use FE-EXP to evaluate the
improvement obtained thanks to the employment of the sliding window. We also compare our method with another non-parametric algorithm named SW-LB-SDA \cite{baudry2021limited}. Furthermore, we compare  with some novel and efficient algorithms such as CUSUM \cite{liu2018change}, M-UCB \cite{cao2019nearly} only in  Bernoulli distribution rewards. Moreover, we compare with  SW-TS \cite{trovo2020sliding}. This method requires information about the  Gaussian rewards  to be known in advance. There is no theoretical proof yet for SW-TS except for Bernoulli rewards. 

\noindent{\bfseries{Tune parameters }} 
Following Remark \ref{remark2}, we set $\tau=\sqrt{T/B_T}\log(T)$ for SW-FE-Exp, $\tau=\sqrt{T\log(T)/B_T}$ for SW-FE-Constant and SW-FE-Linear. 
We set $\tau=\sqrt{T\log(T)/B_T}$ for LB-SDA and SW-TS. For changepoint detection algorithm M-UCB, we set $w=800, b= \sqrt{w/2\log(2KT^2)}$ suggested by \cite{cao2019nearly}. But set the amount of exploration $\gamma = \sqrt{KB_T\log(T)/T}$.  In practice, it has been found that using this value instead of the one guaranteed in \cite{cao2019nearly} will improve empirical performance \cite{baudry2021limited}. For CUSUM, following from \cite{liu2018change}, we set $\alpha= \sqrt{B_T/T\log(T/B_T)}$ and $h=\log(T/B_T)$. For our experiment settings, we choose $M=50, \epsilon=0.05$.

The time horizon is set as $T = 100000$. We split the time horizon into $5$ phases of equal length  and fix the number of arms to $K=5$. Each stationary phase, the reward distributions will be regenerated in the same way as stationary settings.
 Figure \ref{fig:2}  depicts the expected rewards for Gaussian arms
with $K = 5$ and $B_T = 5$. Gaussian and Bernoulli distributions are generated in the same way as in the stationary setting.
The expected regret is averaged on $10$ independently runs. The 95\% confidence interval is obtained by performing $10$
independent runs and is shown as a semi-transparent region in the figure.

Figure \ref{fig:3} shows the results of Gaussian and Bernoulli rewards for piecewise-stationary settings. 
M-UCB and CUMSUM require that the rewards are bounded, which is not applicable to Gaussian rewards. We only conducted experiments on Bernoulli rewards. 
FE-EXP is an algorithm for stationary MAB problems, so it oscillates a lot at the breakpoint. SW-FE-Constant and SW-FE-Linear have similar performance, with SW-FE-Constant even performing better. This could be attributed to the significant impact of problem-dependent term  $m_T=\frac{8\sigma^2}{\Delta_T(i)^2}$ on the performance. The regret upper bound of SW-FE-Constant is controlled by $m_T^2$, while that of SW-FE-Linear is controlled by $m_T^3$.
Our algorithm, SW-FE-EXP,  exhibits competitive results on both Gaussian and Bernoulli rewards.

\section{Conclusion}

In this paper, we have developed a forced exploration algorithm for both stationary and non-stationary multi-armed bandit problems. This algorithm has broad applicability to various reward distributions, and its implementation does not require the use of reward distribution information. We employ a unified analytical approach for different input sequences and provide regret upper bounds. Experimental results demonstrate that despite the asymptotic nature of our regret upper bounds, our approach achieves comparable performance to current popular algorithms.

\nobibliography*

\bibliography{aaai24}

\begin{thebibliography}{32}
\providecommand{\natexlab}[1]{#1}

\bibitem[{Agrawal, Koolen, and Juneja(2021)}]{agrawal2021optimal}
Agrawal, S.; Koolen, W.~M.; and Juneja, S. 2021.
\newblock Optimal best-arm identification methods for tail-risk measures.
\newblock \emph{Advances in Neural Information Processing Systems}, 34:
  25578--25590.

\bibitem[{Ashutosh et~al.(2021)Ashutosh, Nair, Kagrecha, and
  Jagannathan}]{ashutosh2021bandit}
Ashutosh, K.; Nair, J.; Kagrecha, A.; and Jagannathan, K. 2021.
\newblock Bandit algorithms: Letting go of logarithmic regret for statistical
  robustness.
\newblock In \emph{International Conference on Artificial Intelligence and
  Statistics}, 622--630. PMLR.

\bibitem[{Auer, Cesa-Bianchi, and Fischer(2002)}]{auer2002finite}
Auer, P.; Cesa-Bianchi, N.; and Fischer, P. 2002.
\newblock Finite-time analysis of the multiarmed bandit problem.
\newblock \emph{Machine learning}, 47: 235--256.

\bibitem[{Auer, Gajane, and Ortner(2019)}]{auer2019adaptively}
Auer, P.; Gajane, P.; and Ortner, R. 2019.
\newblock Adaptively tracking the best bandit arm with an unknown number of
  distribution changes.
\newblock In \emph{Conference on Learning Theory}, 138--158. PMLR.

\bibitem[{Baransi, Maillard, and Mannor(2014)}]{baransi2014sub}
Baransi, A.; Maillard, O.-A.; and Mannor, S. 2014.
\newblock Sub-sampling for multi-armed bandits.
\newblock In \emph{Machine Learning and Knowledge Discovery in Databases:
  European Conference, ECML PKDD 2014, Nancy, France, September 15-19, 2014.
  Proceedings, Part I 14}, 115--131. Springer.

\bibitem[{Baudry, Russac, and Capp{\'e}(2021)}]{baudry2021limited}
Baudry, D.; Russac, Y.; and Capp{\'e}, O. 2021.
\newblock On limited-memory subsampling strategies for bandits.
\newblock In \emph{International Conference on Machine Learning}, 727--737.
  PMLR.

\bibitem[{Besson et~al.(2022)Besson, Kaufmann, Maillard, and
  Seznec}]{besson2022efficient}
Besson, L.; Kaufmann, E.; Maillard, O.-A.; and Seznec, J. 2022.
\newblock Efficient change-point detection for tackling piecewise-stationary
  bandits.
\newblock \emph{The Journal of Machine Learning Research}, 23(1): 3337--3376.

\bibitem[{Bubeck, Cesa-Bianchi, and Lugosi(2013)}]{bubeck2013bandits}
Bubeck, S.; Cesa-Bianchi, N.; and Lugosi, G. 2013.
\newblock Bandits with heavy tail.
\newblock \emph{IEEE Transactions on Information Theory}, 59(11): 7711--7717.

\bibitem[{Cao et~al.(2019)Cao, Wen, Kveton, and Xie}]{cao2019nearly}
Cao, Y.; Wen, Z.; Kveton, B.; and Xie, Y. 2019.
\newblock Nearly optimal adaptive procedure with change detection for
  piecewise-stationary bandit.
\newblock In \emph{The 22nd International Conference on Artificial Intelligence
  and Statistics}, 418--427. PMLR.

\bibitem[{Chan(2020)}]{chan2020multi}
Chan, H.~P. 2020.
\newblock The multi-armed bandit problem: An efficient nonparametric solution.

\bibitem[{Chen et~al.(2019)Chen, Lee, Luo, and Wei}]{chen2019new}
Chen, Y.; Lee, C.-W.; Luo, H.; and Wei, C.-Y. 2019.
\newblock A new algorithm for non-stationary contextual bandits: Efficient,
  optimal and parameter-free.
\newblock In \emph{Conference on Learning Theory}, 696--726. PMLR.

\bibitem[{Combes and Proutiere(2014)}]{combes2014unimodal}
Combes, R.; and Proutiere, A. 2014.
\newblock Unimodal bandits: Regret lower bounds and optimal algorithms.
\newblock In \emph{International Conference on Machine Learning}, 521--529.
  PMLR.

\bibitem[{Garivier, Lattimore, and Kaufmann(2016)}]{garivier2016explore}
Garivier, A.; Lattimore, T.; and Kaufmann, E. 2016.
\newblock On explore-then-commit strategies.
\newblock \emph{Advances in Neural Information Processing Systems}, 29.

\bibitem[{Garivier and Moulines(2011)}]{garivier2011upper}
Garivier, A.; and Moulines, E. 2011.
\newblock On upper-confidence bound policies for switching bandit problems.
\newblock In \emph{International Conference on Algorithmic Learning Theory},
  174--188. Springer.

\bibitem[{Gorre et~al.(2001)Gorre, Mohammed, Ellwood, Hsu, Paquette, Rao, and
  Sawyers}]{gorre2001clinical}
Gorre, M.~E.; Mohammed, M.; Ellwood, K.; Hsu, N.; Paquette, R.; Rao, P.~N.; and
  Sawyers, C.~L. 2001.
\newblock Clinical resistance to STI-571 cancer therapy caused by BCR-ABL gene
  mutation or amplification.
\newblock \emph{Science}, 293(5531): 876--880.

\bibitem[{Kveton et~al.(2019{\natexlab{a}})Kveton, Szepesvari, Ghavamzadeh, and
  Boutilier}]{kveton2019perturbed}
Kveton, B.; Szepesvari, C.; Ghavamzadeh, M.; and Boutilier, C.
  2019{\natexlab{a}}.
\newblock Perturbed-history exploration in stochastic multi-armed bandits.
\newblock \emph{arXiv preprint arXiv:1902.10089}.

\bibitem[{Kveton et~al.(2019{\natexlab{b}})Kveton, Szepesvari, Vaswani, Wen,
  Lattimore, and Ghavamzadeh}]{kveton2019garbage}
Kveton, B.; Szepesvari, C.; Vaswani, S.; Wen, Z.; Lattimore, T.; and
  Ghavamzadeh, M. 2019{\natexlab{b}}.
\newblock Garbage in, reward out: Bootstrapping exploration in multi-armed
  bandits.
\newblock In \emph{International Conference on Machine Learning}, 3601--3610.
  PMLR.

\bibitem[{Lai, Robbins et~al.(1985)}]{lai1985asymptotically}
Lai, T.~L.; Robbins, H.; et~al. 1985.
\newblock Asymptotically efficient adaptive allocation rules.
\newblock \emph{Advances in applied mathematics}, 6(1): 4--22.

\bibitem[{Lattimore(2017)}]{lattimore2017scale}
Lattimore, T. 2017.
\newblock A scale free algorithm for stochastic bandits with bounded kurtosis.
\newblock \emph{Advances in Neural Information Processing Systems}, 30.

\bibitem[{Li et~al.(2011)Li, Chu, Langford, and Wang}]{li2011unbiased}
Li, L.; Chu, W.; Langford, J.; and Wang, X. 2011.
\newblock Unbiased offline evaluation of contextual-bandit-based news article
  recommendation algorithms.
\newblock In \emph{Proceedings of the fourth ACM international conference on
  Web search and data mining}, 297--306.

\bibitem[{Li, Karatzoglou, and Gentile(2016)}]{li2016collaborative}
Li, S.; Karatzoglou, A.; and Gentile, C. 2016.
\newblock Collaborative filtering bandits.
\newblock In \emph{Proceedings of the 39th International ACM SIGIR conference
  on Research and Development in Information Retrieval}, 539--548.

\bibitem[{Liu, Lee, and Shroff(2018)}]{liu2018change}
Liu, F.; Lee, J.; and Shroff, N. 2018.
\newblock A change-detection based framework for piecewise-stationary
  multi-armed bandit problem.
\newblock In \emph{Proceedings of the AAAI Conference on Artificial
  Intelligence}, volume~32.

\bibitem[{Liu et~al.(2022)Liu, Chen, Deng, and Wu}]{liu2022extended}
Liu, K.; Chen, H.; Deng, W.; and Wu, T. 2022.
\newblock The Extended UCB Policies for Frequentist Multi-armed Bandit
  Problems.
\newblock arXiv:1112.1768.

\bibitem[{Raj and Kalyani(2017)}]{raj2017taming}
Raj, V.; and Kalyani, S. 2017.
\newblock Taming non-stationary bandits: A Bayesian approach.
\newblock \emph{arXiv preprint arXiv:1707.09727}.

\bibitem[{Riou and Honda(2020)}]{riou2020bandit}
Riou, C.; and Honda, J. 2020.
\newblock Bandit algorithms based on thompson sampling for bounded reward
  distributions.
\newblock In \emph{Algorithmic Learning Theory}, 777--826. PMLR.

\bibitem[{Schwartz, Bradlow, and Fader(2017)}]{schwartz2017customer}
Schwartz, E.~M.; Bradlow, E.~T.; and Fader, P.~S. 2017.
\newblock Customer acquisition via display advertising using multi-armed bandit
  experiments.
\newblock \emph{Marketing Science}, 36(4): 500--522.

\bibitem[{Slivkins et~al.(2019)}]{slivkins2019introduction}
Slivkins, A.; et~al. 2019.
\newblock Introduction to multi-armed bandits.
\newblock \emph{Foundations and Trends{\textregistered} in Machine Learning},
  12(1-2): 1--286.

\bibitem[{Thompson(1933)}]{thompson1933likelihood}
Thompson, W.~R. 1933.
\newblock On the likelihood that one unknown probability exceeds another in
  view of the evidence of two samples.
\newblock \emph{Biometrika}, 25(3-4): 285--294.

\bibitem[{Trovo et~al.(2020)Trovo, Paladino, Restelli, and
  Gatti}]{trovo2020sliding}
Trovo, F.; Paladino, S.; Restelli, M.; and Gatti, N. 2020.
\newblock Sliding-window thompson sampling for non-stationary settings.
\newblock \emph{Journal of Artificial Intelligence Research}, 68: 311--364.

\bibitem[{Vermorel and Mohri(2005)}]{vermorel2005multi}
Vermorel, J.; and Mohri, M. 2005.
\newblock Multi-armed bandit algorithms and empirical evaluation.
\newblock In \emph{Machine Learning: ECML 2005: 16th European Conference on
  Machine Learning, Porto, Portugal, October 3-7, 2005. Proceedings 16},
  437--448. Springer.

\bibitem[{Wang et~al.(2020)Wang, Yu, Hao, and Cheng}]{wang2020residual}
Wang, C.-H.; Yu, Y.; Hao, B.; and Cheng, G. 2020.
\newblock Residual bootstrap exploration for bandit algorithms.
\newblock \emph{arXiv preprint arXiv:2002.08436}.

\bibitem[{Wu, Iyer, and Wang(2018)}]{wu2018learning}
Wu, Q.; Iyer, N.; and Wang, H. 2018.
\newblock Learning contextual bandits in a non-stationary environment.
\newblock In \emph{The 41st International ACM SIGIR Conference on Research \&
  Development in Information Retrieval}, 495--504.

\end{thebibliography}

\appendix
\onecolumn
In this appendix, we provide the detailed proof of theorems and corollaries in the main text. We omit the proof of Corollary \ref{corollary4}, Corollary \ref{corollary5} and Corollary \ref{corollary6} as it is similar to the proof of Corollary \ref{corollary1}, Corollary \ref{corollary2}  and Corollary \ref{corollary3}. Beforehand, we give some lemmas to simplify the proof.

\begin{lemma}
    \label{lemma1}
    Assume that $X_i$ are independent, $\sigma$-subgaussian random variables. $\mu=\mathbb{E}[X_i]$. Then for any $\epsilon \geq 0$,
    \[ P(\hat{\mu} \geq \mu +\epsilon ) \leq e^{-\frac{n\epsilon^2}{2\sigma^2}}, \]
    \[ P(\hat{\mu} \leq \mu -\epsilon ) \leq e^{-\frac{n\epsilon^2}{2\sigma^2}}, \]
    where $\hat{\mu}=\frac{1}{n}\sum_{i=1}^{n}X_i$.
\end{lemma}

The following lemma widely used in the analysis of non-stationary  bandits.
\begin{lemma}(\cite{garivier2011upper,combes2014unimodal})
    \label{lemma2}
    For any $i \in \{1,...,K\}$, any integers $\tau$ and $A>0$,
    \[ \sum_{t=1}^{T}\mathds{1}\{i_t=i,N_t(\tau,i) < A\} \leq \lceil T/\tau \rceil A  \]

\end{lemma}

The following lemma helps to bound $h_t(i)$.
\begin{lemma}
    \label{lemma3}
    Define $f^{-1}(r)$ as $\min \{x:f(x)\geq r \} $. Let $t_0$ denote the number of all time steps satisfying $f(r) \leq K$, then $t_0 \leq Kf^{-1}(K+1)$. If $t >t_0$, for any arm $i \in \{1,...,K\}$ ,
    \[ h_t(i) \geq \max \{ n:\sum_{r=f^{-1}(K+1)}^{n}f(r) \leq t-t_0 \}, \] 
    \[ h_T(i) \leq \max \{ n:\sum_{r=f^{-1}(K)}^{n}f(r) \leq T-t_0 \}. \]
\end{lemma}
\begin{proof}
    If $f(r)\leq K$, our method pulls each arm in turn without using the greedy rule at all.
    Before $t_0$ time steps, the number of times each arm is forced to pull is between $f^{-1}(K)$ and  $f^{-1}(K+1)$.
    Then we can get 
    \[ t_0 \leq Kf^{-1}(K+1).\]
    If $t>t_0$, observe that each arm is forced to pull every $f(r)$ time step, so we get the conclusion of this Lemma.
\end{proof}

\section*{Proof of Theorem \ref{theorem1}}
Without loss of generality, assume that the first arm is optimal, i.e., $ \mu(1)= \max_{i \in \{1,...,K\}}\mu(i)$.
The regret incurred by any suboptimal arm may be divided into two components: the greedy rule and forced exploration. We have 
\begin{equation}
    k_T(i) \leq h_T(i)+ \sum_{t=1}^{T} \mathds{1}\{\hat{\mu}_t(i) > \mu(i) + \frac{\Delta(i)}{2} \} + \sum_{t=1}^{T} \mathds{1}\{\hat{\mu}_t(1) \leq \mu(1) - \frac{\Delta(i)}{2} \}.
\end{equation}

Define $\hat{\mu}(i,s)$ to be the empirical mean based on the first $s$ samples. Since $h_t(i)$ denotes the number of times  arm $i$ is forced to pull until time steps $t$, $k_t(i) \geq h_t(i)$. Recall that
$m=\frac{8\sigma^2}{\Delta(i)^2}$, we have 
\begin{equation}
    \begin{aligned}
        \sum_{t=1}^{T} P(\hat{\mu}_t(i) > \mu(i) + \frac{\Delta(i)}{2})&= \sum_{t=1}^{T}  P(\hat{\mu}_t(i) > \mu(i) + \frac{\Delta(i)}{2},k_t(i) \geq h_t(i))\\
        &\leq \sum_{t=1}^{T}\sum_{s=h_t(i)}^{T}P(\hat{\mu}(i,s)> \mu(i)+ \frac{\Delta(i)}{2})\\
        &\leq \sum_{t=1}^{T}\sum_{s=h_t(i)}^{T} e^{-\frac{s}{m}}\\
        &\leq me^{\frac{1}{m}}\sum_{t=1}^{T}e^{-\frac{h_t(i)}{m}}
    \end{aligned}
\end{equation}
The penultimate inequality uses Lemma \ref{lemma1}. The last inequality follows from 
\[
    \sum_{s=h_t(i)}^{T} e^{-\frac{s}{m}} \leq \frac{e^{-\frac{h_t(i)}{m}}}{1-e^{-\frac{1}{m}}} \leq me^{\frac{1}{m}}e^{-\frac{h_t(i)}{m}} 
\]
Therefore,
\begin{equation}
    \label{temp2}
    \mathbb{E}[k_T(i)] \leq  h_T(i) 
+ 2m e^{\frac{1}{m}} \sum_{t=1}^{T} e^{\frac{-h_t(i)}{m}}. 
\end{equation}
So far, we have completed the proof of Theorem \ref{theorem1}. 

Let $t_1$ denotes some time steps. If we only focus on the regret after $t_1$, according to the above analysis, we have
\begin{equation}
    \label{temp3}
    \mathbb{E}[k_T(i)] \leq  h_T(i) +t_1 
+ 2m e^{\frac{1}{m}} \sum_{t=t_1+1}^{T} e^{\frac{-h_t(i)}{m}}. 
\end{equation}

\section*{Proof of Corollary \ref{corollary1}}

Following from Lemma \ref{lemma3},   if $t>t_0$ 
\[ h_t(i) \geq \max \{ n:\sum_{r=f^{-1}(K+1)}^{n}f(r) \leq t-t_0 \}, \] 
\[ h_T(i) \leq \max \{ n:\sum_{r=f^{-1}(K)}^{n}f(r) \leq T-t_0 \}. \]
Since $f(r)=\sqrt{T}$ and $K \ll T$, we can get $f^{-1}(K)=0$ and $t_0=0$. This implies 
\[
    h_t(i)=\max \{n: \sum_{r=0}^{n}f(r) \leq t \} \in [ \frac{t}{\sqrt{T}}-1,\frac{t}{\sqrt{T}}+1].
\]
Therefore, $h_T(i) \leq \sqrt{T}+1$.
\[ 
    \sum_{t=1}^{T} e^{-\frac{h_t(i)}{m}} \leq e^{\frac{1}{m}}\sum_{t=1}^{T}e^{-\frac{t}{\sqrt{T}m}} \leq  \sqrt{T}me^{\frac{1}{m}} 
\]
Substituting, we get,
\[ 
    \mathbb{E}[k_T(i)] \leq \sqrt{T}(1+2m^2e^{\frac{2}{m}})+1.
\]

\section*{Proof of Corollary \ref{corollary2}}
For a linear sequence $f(r)=r$, we have $f^{-1}(r)=r, t_0=Kf^{-1}(K)=K^2$. To bound $h_T(i)$, 
let 
\[\sum_{r=K}^{n}r=T-K^2,\]
we can get $ n^2+n=2T-K^2-K$.  Hence, 
\[ h_T(i) \leq \sqrt{2T}. \]
Let \[\sum_{r=K+1}^{n}r=t-K^2,\]
then $n^2+n=2t -K^2+K$. 
When $t > K^2-K$, $ n \geq \sqrt{t}-1$. We have, 
\[ h_t(i) \geq n \geq \sqrt{t}-1,\] 
\[ 
    \sum_{t=K^2+1}^{T} e^{-\frac{h_t(i)}{m}} \leq e^{\frac{1}{m}}\sum_{t=1}^{T}e^{-\frac{\sqrt{t}}{m}} \leq e^{\frac{1}{m}} \frac{3e^{-\frac{1}{m}}}{(1-e^{-\frac{1}{m}})^2} \leq 3e^{\frac{2}{m}}m^2.
\]
The penultimate inequality uses the fact that ($0<q<1$)
\begin{equation*}
    \begin{aligned}
        &\sum_{n=1}^{\infty}q^{\sqrt{n}}\\
        &\leq \sum_{n=1}^{\infty}q^{[\sqrt{n}]}=\sum_{k=1}^{\infty}\sum_{n=k^2}^{(k+1)^2-1}q^k\\
        &=\sum_{k=1}^{\infty}(2k+1)q^k=\frac{3q-q^2}{(1-q)^2}
    \end{aligned}
\end{equation*}
Let $t_1=K^2$  and substitute  in Equation \ref{temp3}, we have
\[ 
    \mathbb{E}[k_T(i)] \leq \sqrt{2T} +K^2 
    +6m^3e^{\frac{3}{m}}
\]

\section*{Proof of Corollary \ref{corollary3}}
For an exponential sequence $f(r)=a^r$, we have $f^{-1}(r)=\frac{\log(r)}{\log(a)},t_0 \leq K\frac{\log(K+1)}{\log(a)}$. 
From Lemma \ref{lemma3}, $h_T(i)$ can be bounded as 
\[
    h_T(i) \leq  \max\{n: \sum_{r=0}^{n}f(r) \leq T \}+ f^{-1}(K+1) \leq \frac{\log(T(a-1)+1)}{\log(a)} +\frac{\log(K+1)}{\log(a)}.
\]
To estimate $\sum_{t=t_0+1}^{T} e^{\frac{-h_t(i)}{m}}$, we need to compute a lower bound on $h_t(i)$.
Let 
\[
    \sum_{r=\log(K+1)/\log(a)}^{n}a^r = t-t_0.
\] 
After some calculation, we have 
\[ 
    h_t(i)  \geq \frac{\log(1+\frac{(t-t_0)(a-1)}{K+1})}{\log(a)}
\]
Let $q=e^{-\frac{1}{m}}$,
\[
    \begin{aligned}
    \sum_{t=t_0+1}^{T} e^{\frac{-h_t(i)}{m}}  &= \sum_{t=t_0+1}^{T} q^{h_t(i)} \\
    &\leq \sum_{t=t_0+1}^{T}q^{\frac{\log(1+\frac{(t-t_0)(a-1)}{K+1})}{\log(a)}} \\
    &=\sum_{t=t_0+1}^{T} (1+\frac{(t-t_0)(a-1)}{K+1})^{\frac{\log(q)}{\log(a)}}\\
    &\leq \sum_{t=1}^{T}(1+\frac{a-1}{K+1}t)^{\frac{\log(q)}{\log(a)}} \\
    &= \sum_{t=1}^{T} (1+\frac{a-1}{K+1}t)^{-\frac{1}{m\log(a)}}
    \end{aligned} 
\]

Substituting, we get,
\[ \mathbb{E}[k_T(i)] \leq \frac{\log(T(a-1)+1)}{\log(a)} +(K+1)\frac{\log(K+1)}{\log(a)} + 2me^{\frac{1}{m}}\sum_{t=1}^{T} (1+\frac{a-1}{K+1}t)^{-\frac{1}{m\log(a)}} \]

\section*{Proof of Theorem \ref{theorem2}}

Define $\mathcal{T}(\tau)= \{ t\leq T: \forall s \in (t-\tau,t), \mu_s(\cdot)=\mu_t(\cdot) \}.$ 
$\mathcal{T}(\tau)$ can be seen as the stationary phase, but it is shorter than the actual stationary phase since there will be $\tau$ rounds near the  breakpoints that do not belong to $\mathcal{T}(\tau)$. 

Let $A(\tau)= 2m_T\log(\tau)$. Recall that, $h_t(\tau,i)$ denotes the number of forced pulls for arm $i$ in the $\tau$ last plays. $h_T(T,i)$ denotes the number of forced pulls for arm $i$ in the whole time horizon.
We upper-bound the number of times the suboptimal arm $i$ is played as follows:
\begin{equation*}
    \begin{aligned}
    k_T(i) &= \sum_{t=1}^{T} \mathds{1}\{i_t =i \neq i_t^{*} \}\\
    &\leq  \sum_{t=1}^{T} \mathds{1}\{ i_t=i \neq i_t^{*}, N_t(\tau,i) \leq A(\tau)\} + \sum_{t=1}^{T} \mathds{1}\{ i_t=i \neq i_t^{*}, N_t(\tau,i) > A(\tau) \} +h_T(T,i) \\
    &\leq \lceil T/\tau \rceil A(\tau)+\sum_{t=1}^{T} \mathds{1}\{ i_t=i \neq i_t^{*}, N_t(\tau,i) > A(\tau) \} +h_T(T,i) \\
    &\leq  \lceil T/\tau \rceil A(\tau)+B_T\tau +  \sum_{t \in \mathcal{T}(\tau)} \mathds{1}\{ i_t=i \neq i_t^{*}, N_t(\tau,i) > A(\tau) \} +h_T(T,i)
    \end{aligned}
\end{equation*}
For $t \in \mathcal{T}(\tau)$, we have
\[ 
    \begin{aligned}
    \{ i_t=i \neq i_t^{*}, N_t(\tau,i) > A(\tau) \}  
    \subset \{ \mu_t(\tau,i) > \mu_t(i) + \frac{\Delta(i)}{2}, N_t(\tau,i)>A(\tau)\} \cup \{\mu_t(\tau,i_t^{*}) \leq \mu_t(i_t^{*}) -\frac{\Delta(i)}{2} \}
    \end{aligned} 
\]
The first part can be bounded as follows:
\[
\begin{aligned}
    &\sum_{t=1}^{T}P(\mu_t(\tau,i) > \mu_t(i) + \frac{\Delta(i)}{2}, N_t(\tau,i)>A(\tau)) \\
    &\leq \sum_{t=1}^{T}\sum_{s=A(\tau)}^{\tau}P(\mu_t(i,s)-\mu_t(i) >\frac{\Delta(i)}{2})\\
    &\leq \sum_{t=1}^{T}\sum_{s=A(\tau)}^{\tau} e^{-\frac{s}{m_T}} \\
    &\leq \frac{T}{\tau}
\end{aligned}    
\]
Recall that, $h_t(\tau,i)$ denotes the number of forced pulls for arm $i$ in the $\tau$ last plays. 
\[
\begin{aligned}
    &\sum_{t=1}^{T} P(\mu_t(\tau,i_t^{*}) \leq \mu_t(i_t^{*}) -\frac{\Delta(i)}{2}) \\
    &\leq \sum_{t=1}^{T} P(\mu_t(\tau,i_t^{*}) \leq \mu_t(i_t^{*}) -\frac{\Delta(i)}{2},N_t(\tau,i_t^{*}) \geq h_t(\tau,i)) \\
    &\leq \sum_{t=1}^{T}\sum_{s=h_t(\tau,i)}^{\tau}e^{-\frac{s}{m_T}}\\
    &\leq m_Te^{\frac{1}{m_T}}\sum_{t=1}^{T}e^{-\frac{h_t(\tau,i)}{m_T}}
\end{aligned}    
\]
Our algorithm restarts the exploration sequence $\{f(r)\}$ every $\tau $ time steps. We can  only analyze the time steps in $[1,\tau]$. Then multiply the regret between $(1,\tau)$ by $\frac{T}{\tau}$  yields the total regret. 

This rough estimation increases the regret upper bound. To get an intuition, consider the trend of $h_t(\tau,i)$ on$ (1,\tau)$ versus $(\tau,2\tau)$. On $(1,\tau)$, $h_t(\tau,i)$  is increasing from 0, while on $(\tau,2\tau)$, $h_t(\tau,i)$ is changing from $h_{\tau}(\tau,i)$. This  underestimates $h_t(\tau,i)$ and thus increases the regret upper bound.

We have 
\[ 
    \sum_{t=1}^{T} e^{-\frac{h_t(\tau,i)}{m_T}} \leq \frac{T}{\tau}
    \sum_{t=1}^{\tau} e^{-\frac{h_t(\tau,i)}{m_T}},
\]
\[h_T(T,i)= \frac{T}{\tau}\sum_{t=1}^{\tau} h_{\tau}(\tau,i).\]
Substituting, we get,
\[ 
    \mathbb{E}[k_T(i)] \leq \frac{T}{\tau}\left (h_{\tau}(\tau,i)
    + m_T e^{\frac{1}{m_T}} \sum_{t=1}^{\tau} e^{\frac{-h_t(\tau,i)}{m_T}}\right ) +\frac{T}{\tau}(1+2m_T\log(\tau))+B_T\tau.
\]

\end{document}